\documentclass[letterpaper]{article} 
\usepackage{aaai2026}  
\usepackage{times}  
\usepackage{helvet}  
\usepackage{courier}  
\usepackage[hyphens]{url}  
\usepackage{graphicx} 
\usepackage{natbib}   
\usepackage{caption} 
\usepackage{algorithm}
\usepackage{algorithmic}

\usepackage{newfloat}
\usepackage{listings}
\DeclareCaptionStyle{ruled}{labelfont=normalfont,labelsep=colon,strut=off} 
\floatstyle{ruled}
\newfloat{listing}{tb}{lst}{}
\floatname{listing}{Listing}
\usepackage{amsmath}
\usepackage{amssymb}
\usepackage{amsthm}
\usepackage{stmaryrd}
\usepackage{mathtools}
\usepackage{enumitem}
\usepackage{verbatim}

\theoremstyle{plain}
\newtheorem{theorem}{Theorem}
\newtheorem{proposition}[theorem]{Proposition}
\newtheorem{lemma}[theorem]{Lemma}
\newtheorem{corollary}[theorem]{Corollary}

\theoremstyle{definition}
\newtheorem{definition}{Definition}

\theoremstyle{remark}

\newcommand{\Vfour}{\mathbb{V}_4}
\newcommand{\truev}{\mathbf{t}}
\newcommand{\falsev}{\mathbf{f}}
\newcommand{\nonev}{\mathbf{n}}
\newcommand{\bothv}{\mathbf{b}}
\newcommand{\state}{\mathsf{state}}
\newcommand{\countop}{\mathsf{count}}

\title{A Four-Valued Normative Intermediate Representation for ASP-Oriented Compliance Reasoning}

\author{
    Huanyu Yang,\equalcontrib\textsuperscript{\rm 1}
    Yangfan Wu,\equalcontrib\textsuperscript{\rm 2}
    Jianmin Ji\textsuperscript{\rm 1}\thanks{Corresponding author.}
}

\affiliations{
    \textsuperscript{\rm 1}School of Computer Science and Technology, University of Science and Technology of China (USTC),\\
    Hefei, Anhui, China\\
    \textsuperscript{\rm 2}Department of Computer Science and Engineering, The Hong Kong University of Science and Technology (HKUST),\\
    Hong Kong SAR, China\\
    yanghuanyu@mail.ustc.edu.cn,
    yangfan.wu@connect.ust.hk,
    jianmin@ustc.edu.cn,
}
\begin{document}

\maketitle

\begin{abstract}
Technical-standard compliance reasoning may involve incomplete evidence,
inconsistent observations, exceptions, and derived normative outputs.  This
paper presents \textsc{Monir}, a four-valued normative intermediate
representation for ASP-oriented compliance workflows. \textsc{Monir}
separates factual support from normative outputs: facts are represented by
positive and negative support bits, while deontic labels are rule-generated
output constructors interpreted by aggregation policies.  Its core semantics
is a staged transition system over support configurations.  
We define the
language, rule-state interface, diagnostics, and policy-parametric reporting;
characterize admissible staged evaluation by a dependency graph; prove
deterministic polynomial-time evaluation for fixed evidence; and give NP/coNP
upper bounds, with matching hardness for representative completion-based
verdict queries.
\end{abstract}

\section{Introduction}
\label{sec:introduction}

Technical-standard compliance checking often takes place before final testing
or certification.  At that stage, design evidence may be incomplete,
partially reviewed, or inconsistent across documents, simulations, and expert
annotations.  A useful compliance representation should therefore preserve
evidence assumptions, exception structure, generated normative conclusions,
and diagnostic provenance, rather than collapse them into a single textual
judgment.

Formal approaches to compliance checking and normative reasoning have studied
conflicting and compensatory norms, deontic rule systems, and ASP-based
implementations of normative logics~\cite{robaldo2024compliance,
governatori2024asp}.  These works motivate the use of formal reasoning for
compliance tasks.  However, directly encoding a technical standard in the
backend logic can mix extraction choices, evidence status, exception
modelling, and solver-level predicate design in one artifact.  This makes
review and localization of modelling errors difficult when the source standard
or the available evidence changes.

Large language models can assist legal and regulatory extraction, but their
outputs are not authoritative compliance judgments
~\cite{guha2023legalbench,dahl2024large}.  Here LLM-assisted modelling is an
authoring route into a typed intermediate representation.  The backend
evaluates the resulting \textsc{Monir} rule base, not an LLM-generated textual
answer.

\textsc{Monir} is designed as a normative intermediate representation between
technical-standard text and executable ASP-based checking.  Its design
choice is to separate three layers.  First, factual evidence is represented by
positive and negative support, giving four readings: positive only, negative
only, no support, and conflicting support.  Second, deontic labels such as
$\mathsf O$ and $\mathsf F$ are output constructors generated by named rules;
they are not truth-valued modal operators in the factual language.  Third,
compliance-level conclusions are produced by read-only aggregation policies
over saturated support configurations, issued outputs, and diagnostics.

This separation keeps evidence conflicts, output conflicts, exemptions, and
contrary-to-duty outputs explicit, while leaving ASP as an executable target
rather than the reference semantics.

The paper makes two contributions.
\begin{enumerate}
\item It defines \textsc{Monir}, a four-valued normative intermediate
representation with explicit support, rule states, diagnostics, and
policy-parametric reporting.
\item It gives an ASP-oriented executable account, polynomial fixed-evidence
evaluation, NP/coNP bounds for completion-based verdict queries, and a
projection-correctness condition for demand slicing.
\end{enumerate}

\section{Preliminaries}
\label{sec:preliminaries}

This section fixes the background for \textsc{Monir}: four-valued support,
normative outputs, ASP execution, and structured standard modelling.

\subsection{Four-Valued Evidence and Normative Outputs}
\label{subsec:four-valued-evidence}

Belnap--Dunn four-valued semantics provides a standard account of incomplete
and inconsistent information~\cite{belnap1977useful}.  We use the value set
$\Vfour=\{\truev,\falsev,\nonev,\bothv\}$ with a support-based reading:
$\truev$ denotes positive support only, $\falsev$ negative support only,
$\nonev$ no support, and $\bothv$ both positive and negative support.

The four-valued layer is used only for factual support.  Explicit factual
negation is interpreted by reversing support polarity, as formalized in
Definition~\ref{def:literal-evaluation}.  Rule arrows are constructors, not
object-level implications in the four-valued language.  A factual rule
$n:\kappa\leadsto\ell$ adds support for a factual literal, while a normative
rule $n:\alpha\Rightarrow\omega$ generates a named normative output.  Deontic
labels in $\omega$ are output constructors rather than truth-valued modal
operators; compliance-level interpretation is delegated to aggregation
policies.

\subsection{ASP Backend and Modelling Interface}
\label{subsec:backend-target-modelling}

Answer Set Programming is the intended executable backend.  We rely on the
stable-model view of logic programs~\cite{gelfond1988stable}.
Cardinality and weight constructs are relevant because \textsc{Monir} includes
bounded support-count expressions; their ASP realization may use standard
aggregate constructs or auxiliary encodings when appropriate.  The semantics
below is nevertheless defined independently of any particular ASP encoding.

The executable realization instantiates structured \textsc{Monir} records as
ASP facts and rule templates for clingo, whose multi-shot interface supports
controlled execution~\cite{gebser2019multi}.  Clingo is an implementation
vehicle, not the source of the reference semantics.

The authoring interface is not fixed.  Structured \textsc{Monir} records may
be produced manually or by an LLM-assisted pipeline after extraction,
normalization, exception handling, and review.  The input text need not be a
controlled natural language; \textsc{Monir} fixes the restricted,
machine-checkable rule format used after modelling.

\section{\textsc{Monir}-core}
\label{sec:monir-core}

This section defines the \textsc{Monir} language, support-configuration
semantics, and admissible staged evaluation.  Stored configurations contain
support bits and diagnostics; rule states, issued outputs, and reports are
derived.

\subsection{Language and Syntactic Notation}
\label{subsec:monir-language}

Let $\mathcal P$ be a finite set of \emph{factual atoms}, for which evidence
is stored as positive or negative support.  Let $\mathcal N$ be a finite set
of \emph{rule identifiers}, used for provenance, rule-state atoms, and
exemption outputs.
The \emph{deontic output labels} are
$\mathcal M=\{\mathsf O,\mathsf F,\mathsf P,\mathsf R,\mathsf{NR}\}$.
They stand for obligation, prohibition, permission, recommendation, and
recommendation against a target.  The \emph{hard deontic labels} are
$\mathcal M_{\mathsf h}=\{\mathsf O,\mathsf F\}$.
All deontic labels are output constructors, not truth-valued modal operators.
The \emph{rule-state labels} are
\[
\begin{aligned}
\mathcal S=\{&
\mathsf{pending},\mathsf{out},\mathsf{applicable},
\mathsf{contested},\mathsf{exempted},\\
&\mathsf{issued},\mathsf{fulfilled},\mathsf{violated},
\mathsf{unknown}\}.
\end{aligned}
\]
A \emph{rule-state atom} is an expression $\state(s,n)$, where
$s\in\mathcal S$ and $n\in\mathcal N$.  Rule-state atoms are derived query
atoms.  They are not directly assigned by transitions.

\begin{definition}[\textsc{Monir} expressions]
\label{def:monir-expressions}
The expressions of \textsc{Monir} are generated by
\[
\begin{aligned}
\ell &::= p \mid \neg p,\qquad
b ::= \ell \mid \state(s,n),\\
\kappa &::= \top \mid \ell
    \mid \kappa_1\wedge\kappa_2
    \mid \countop_{[l,u]}(\ell_1,\ldots,\ell_m),\\
\alpha &::= \top \mid b
    \mid \alpha_1\wedge\alpha_2
    \mid \countop_{[l,u]}(b_1,\ldots,b_m),\\
\tau &::= \ell
    \mid \tau_1\wedge\tau_2
    \mid \countop_{[l,u]}(\ell_1,\ldots,\ell_m),\\
\omega &::= M(\tau) \mid \mathsf{exempt}(n),\\
\rho &::= n:\kappa\leadsto\ell,\qquad
r ::= n:\alpha\Rightarrow\omega .
\end{aligned}
\]
$p\in\mathcal P$, $n\in\mathcal N$, $s\in\mathcal S$, and
$M\in\mathcal M$.  Each count expression satisfies
$m\geq 1$ and $0\leq l\leq u\leq m$.  Count arguments are assumed pairwise
distinct.  All expressions are finite.  Let $\Omega$ be the set of output
expressions generated by the $\omega$ grammar.
\end{definition}

A \emph{factual condition} $\kappa$ contains only factual literals and
support-count expressions over factual literals.  A \emph{normative
condition} $\alpha$ may additionally contain rule-state atoms and
support-count expressions over body atoms.  A \emph{factual target} $\tau$ is
the factual content to which a deontic output label is attached.
For a rule $q$ of the form $n:\cdots$, let $\mathsf{id}(q)=n$.  For a rule
set $R$, write
$\mathsf{id}(R)=\{\mathsf{id}(q)\mid q\in R\}$.

\begin{definition}[Well-formed rule base]
\label{def:well-formed-rule-base}
A \textsc{Monir} \emph{well-formed rule base} is a pair
$\mathcal B=\langle R_P,R_N\rangle$, where $R_P$ is a finite set of factual
rules and $R_N$ is a finite set of normative rules.  Write
$R_{\mathcal B}=R_P\cup R_N$.  The following conditions hold:
\begin{enumerate}
\item If $q,q'\in R_{\mathcal B}$ and $q\neq q'$, then
$\mathsf{id}(q)\neq\mathsf{id}(q')$.
\item Every rule identifier occurring in a normative-body state atom or in an
exemption output belongs to $\mathsf{id}(R_N)$.
\end{enumerate}
\end{definition}

Unless stated otherwise, rule bases considered below are well formed.

For a factual rule $\rho=n:\kappa\leadsto\ell$, define
$\mathsf{body}(\rho)=\kappa$,
$\mathsf{head}(\rho)=\ell$ .
For a normative rule $r=n:\alpha\Rightarrow\omega$, define
$\mathsf{body}(r)=\alpha$, 
$\mathsf{out}(r)=\omega$.
For a well-formed rule base $\mathcal B=\langle R_P,R_N\rangle$, define
\[
\begin{aligned}
&R_{\mathsf h} =
\{r\in R_N\mid
  \mathsf{out}(r)=M(\tau)
  \text{ for some } M\in\mathcal M_{\mathsf h} \text{ and } \tau\},\\
&R_{\mathsf{ex}} =
\{r\in R_N\mid
  \mathsf{out}(r)=\mathsf{exempt}(n)
  \text{ for some } n\in\mathsf{id}(R_N)\}.
\end{aligned}
\]
For each $n\in\mathsf{id}(R_N)$, introduce fresh support carriers
$c_n$ and $e_n$.  For each $n\in\mathsf{id}(R_{\mathsf h})$, introduce a
fresh support carrier $m_n$.  These carriers are pairwise distinct and
disjoint from $\mathcal P$.
By well-formedness, each rule identifier denotes at most one rule.  If
$\mathsf{id}(q)=n$, we write $\mathsf{body}(n)$ for $\mathsf{body}(q)$ and,
when $q\in R_N$, $\mathsf{out}(n)$ for $\mathsf{out}(q)$.  We also write
$c_q,e_q,m_q$ for $c_n,e_n,m_n$, respectively, whenever the corresponding
carriers exist.
For factual literals, set
\[
\mathsf{atm}(p)=\mathsf{atm}(\neg p)=p,
\quad
\mathsf{pol}(p)=+,
\quad
\mathsf{pol}(\neg p)=-.
\]
Let $\overline{+}=-$ and $\overline{-}=+$.  Write $\ell^\circ$ for the
complementary literal, where $p^\circ=\neg p$ and $(\neg p)^\circ=p$.

Define $\mathsf{Lit}(e)$ and $\mathsf{St}(e)$ structurally as the sets of
factual literals and rule-state atoms occurring in $e$, respectively.
For an expression $e$, let $\mathsf{VLit}(e)$ be the set of factual literals
whose support may affect the value of $e$:
\[
\begin{aligned}
&\mathsf{VLit}(\top) = \emptyset,\
\mathsf{VLit}(\ell) = \{\ell,\ell^\circ\},\
\mathsf{VLit}(\state(s,n)) = \emptyset,\\
&\mathsf{VLit}(e_1\wedge e_2)
=
\mathsf{VLit}(e_1)\cup\mathsf{VLit}(e_2),\\
&\mathsf{VLit}(\countop_{[l,u]}(x_1,\ldots,x_m))
=
\bigcup_{1\leq i\leq m}\mathsf{VLit}(x_i).
\end{aligned}
\]
Let
$\mathsf{VAt}(e)=\{\mathsf{atm}(\ell)\mid \ell\in\mathsf{VLit}(e)\}$.
For an expression $e$, let $\mathsf{ULit}(e)$ be the set of factual literals
occurring inside non-vacuously upper-bounded count subexpressions:
\[
\begin{aligned}
&\mathsf{ULit}(\top) = \emptyset,\
\mathsf{ULit}(\ell) = \emptyset,
\ \mathsf{ULit}(\state(s,n)) = \emptyset,\\
&\mathsf{ULit}(e_1\wedge e_2)
=
\mathsf{ULit}(e_1)\cup\mathsf{ULit}(e_2),\\
&\mathsf{ULit}(\countop_{[l,u]}(x_1,\ldots,x_m))
=
\begin{cases}
\displaystyle\bigcup_{1\leq i\leq m}\mathsf{Lit}(x_i), & u<m,\\
\emptyset, & u=m .
\end{cases}
\end{aligned}
\]
Let
$\mathsf{UAt}(e)=\{\mathsf{atm}(\ell)\mid \ell\in\mathsf{ULit}(e)\}$.

\subsection{Support Configurations and Expression Evaluation}
\label{subsec:support-configurations-evaluation}

The core transition system operates on support configurations.  A
configuration stores support bits and diagnostic records; rule states, issued
outputs, and expression readings are derived from it.

\begin{definition}[Support configurations]
\label{def:support-configurations}
Define
\[
\begin{aligned}
&A_{\mathcal B}^{\pm}
=
\mathcal P
\cup
\{c_n\mid n\in\mathsf{id}(R_N)\}
\cup
\{m_n\mid n\in\mathsf{id}(R_{\mathsf h})\},
\\
&A_{\mathcal B}^{+}
=
\{e_n\mid n\in\mathsf{id}(R_N)\}.
\end{aligned}
\]
The support-bit vocabulary is
\[
\mathsf{Bit}(\mathcal B)
=
(\{+,-\}\times A_{\mathcal B}^{\pm})
\cup
(\{+\}\times A_{\mathcal B}^{+}).
\]
The diagnostic vocabulary is
\[
\begin{aligned}
\mathsf{Diag}(\mathcal B)
=
&\{\mathsf{conf}(a)\mid a\in A_{\mathcal B}^{\pm}\}\\
&\cup
\{\mathsf{flow}(q,a)
  \mid q\in R_{\mathcal B},\
  a\in A_{\mathcal B}^{\pm}\cup A_{\mathcal B}^{+}\}.
\end{aligned}
\]
A support configuration is a pair
$\Gamma=\langle S,D\rangle$, where
$S\subseteq\mathsf{Bit}(\mathcal B)$ and
$D\subseteq\mathsf{Diag}(\mathcal B)$.

For a well-formed bit $(\epsilon,a)$, write
$\Gamma\Vdash^\epsilon a$ iff
$(\epsilon,a)\in S$.
For $a\in A_{\mathcal B}^{\pm}$, set
$\nu_\Gamma(a)
=
\mathsf{val}(\Gamma\Vdash^+a,\Gamma\Vdash^-a)$,
where
\[
\mathsf{val}(P,N)=
\begin{cases}
\truev  & P\wedge\neg N,\\
\falsev & \neg P\wedge N,\\
\bothv  & P\wedge N,\\
\nonev  & \neg P\wedge\neg N.
\end{cases}
\]
\end{definition}

\begin{definition}[Literal evaluation]
\label{def:literal-evaluation}
For $\Gamma=\langle S,D\rangle$ and a factual literal $\ell$, define
\[
\Gamma\Vdash^\epsilon\ell
\iff
\begin{cases}
(\mathsf{pol}(\ell),\mathsf{atm}(\ell))\in S, & \epsilon=+,\\
(\overline{\mathsf{pol}(\ell)},\mathsf{atm}(\ell))\in S, & \epsilon=-.
\end{cases}
\]
Thus explicit negation reverses support polarity.
\end{definition}

\begin{definition}[Derived rule-state relation]
\label{def:derived-rule-state-relation}
For $r\in R_N$, write
$\Gamma\Vdash_O r$
iff
$\Gamma\Vdash^+c_r
\wedge
\Gamma\not\Vdash^+e_r$.
The derived rule-state relation is defined as follows.  For $r\in R_N$,
\[
\begin{aligned}
&\Gamma\Vdash_S\state(\mathsf{pending},\mathsf{id}(r))
\iff
\nu_\Gamma(c_r)=\nonev,\\
&\Gamma\Vdash_S\state(\mathsf{out},\mathsf{id}(r))
\iff
\nu_\Gamma(c_r)=\falsev,\\
&\Gamma\Vdash_S\state(\mathsf{applicable},\mathsf{id}(r))
\iff
\nu_\Gamma(c_r)=\truev,\\
&\Gamma\Vdash_S\state(\mathsf{contested},\mathsf{id}(r))
\iff
\nu_\Gamma(c_r)=\bothv,\\
&\Gamma\Vdash_S\state(\mathsf{exempted},\mathsf{id}(r))
\iff
\Gamma\Vdash^+e_r,\\
&\Gamma\Vdash_S\state(\mathsf{issued},\mathsf{id}(r))
\iff
\Gamma\Vdash_O r .
\end{aligned}
\]
For $r\in R_{\mathsf h}$, add
\[
\begin{aligned}
&\Gamma\Vdash_S\state(\mathsf{fulfilled},\mathsf{id}(r))
\iff
\Gamma\Vdash_O r\wedge\Gamma\Vdash^+m_r,\\
&\Gamma\Vdash_S\state(\mathsf{violated},\mathsf{id}(r))
\iff
\Gamma\Vdash_O r\wedge\Gamma\Vdash^-m_r,\\
&\Gamma\Vdash_S\state(\mathsf{unknown},\mathsf{id}(r))
\iff
\Gamma\Vdash_O r\wedge\nu_\Gamma(m_r)=\nonev.
\end{aligned}
\]
The relation is not functional: a contested rule may still be issued, and an
issued obligation or prohibition may be both fulfilled and violated.
\end{definition}

\begin{definition}[Counting tests]
\label{def:counting-tests}
For a sequence $\vec e=(e_1,\ldots,e_m)$, define
\[
\begin{aligned}
&P_\Gamma(\vec e)
=
|\{i\mid \Gamma\Vdash^+e_i\}|,\
N_\Gamma(\vec e)
=
|\{i\mid \Gamma\Vdash^-e_i\}|,\\
&\mathsf{pos}_{[l,u]}^\Gamma(\vec e)
\iff
l\leq P_\Gamma(\vec e)\leq u,\\
&\mathsf{neg}_{[l,u]}^\Gamma(\vec e)
\iff
P_\Gamma(\vec e)>u
\vee
m-N_\Gamma(\vec e)<l.
\end{aligned}
\]
For a body-atom sequence $\vec b=(b_1,\ldots,b_m)$, $P_\Gamma(\vec b)$ and
$N_\Gamma(\vec b)$ are defined in the same way, using
$\Gamma\Vdash_B^\pm b_i$ instead of $\Gamma\Vdash^\pm b_i$.
\end{definition}

\begin{definition}[Evaluation of factual conditions and deontic targets]
\label{def:evaluation-factual-targets}
Literal cases are given by Definition~\ref{def:literal-evaluation}.  For
$x,x_1,x_2$ ranging over factual conditions or deontic targets, define
\[
\begin{aligned}
&\Gamma\Vdash^\epsilon(x_1\wedge x_2)
\iff
\begin{cases}
\Gamma\Vdash^+x_1\wedge\Gamma\Vdash^+x_2, & \epsilon=+,\\
\Gamma\Vdash^-x_1\vee\Gamma\Vdash^-x_2, & \epsilon=-,
\end{cases}\\
&\Gamma\Vdash^\epsilon\countop_{[l,u]}(\vec\ell)
\iff
\begin{cases}
\mathsf{pos}_{[l,u]}^\Gamma(\vec\ell), & \epsilon=+,\\
\mathsf{neg}_{[l,u]}^\Gamma(\vec\ell), & \epsilon=-.
\end{cases}
\end{aligned}
\]
In addition, $\Gamma\Vdash^+\top$ holds and $\Gamma\Vdash^-\top$ does not
hold for factual conditions.
\end{definition}

\begin{definition}[Evaluation of normative bodies]
\label{def:evaluation-normative-bodies}
For body atoms, define
\[
\begin{aligned}
&\Gamma\Vdash_B^+\ell \iff \Gamma\Vdash^+\ell,
\quad
\Gamma\Vdash_B^-\ell \iff \Gamma\Vdash^-\ell,\\
&\Gamma\Vdash_B^+\state(s,n) \iff \Gamma\Vdash_S\state(s,n),
\\
&\Gamma\Vdash_B^-\state(s,n) \iff \bot .
\end{aligned}
\]
Normative bodies are evaluated by
\[
\begin{aligned}
&\Gamma\Vdash^\epsilon b
\iff
\Gamma\Vdash_B^\epsilon b,\ \ 
\Gamma\Vdash^\epsilon\top
\iff
\begin{cases}
\top, & \epsilon=+,\\
\bot, & \epsilon=-,
\end{cases}\\
&\Gamma\Vdash^\epsilon(\alpha_1\wedge\alpha_2)
\iff
\begin{cases}
\Gamma\Vdash^+\alpha_1\wedge\Gamma\Vdash^+\alpha_2, & \epsilon=+,\\
\Gamma\Vdash^-\alpha_1\vee\Gamma\Vdash^-\alpha_2, & \epsilon=-,
\end{cases}\\
&\Gamma\Vdash^\epsilon\countop_{[l,u]}(\vec b)
\iff
\begin{cases}
\mathsf{pos}_{[l,u]}^\Gamma(\vec b), & \epsilon=+,\\
\mathsf{neg}_{[l,u]}^\Gamma(\vec b), & \epsilon=-.
\end{cases}
\end{aligned}
\]
Thus state atoms contribute only positively.
\end{definition}

For every expression $e$ for which positive and negative support relations are
defined, write
$\nu_\Gamma(e)
=
\mathsf{val}(\Gamma\Vdash^+e,\Gamma\Vdash^-e)$.

\subsection{Rule-Restricted Transitions and Staged Saturation}
\label{subsec:staged-support-semantics}

The transition rules are parameterized by a rule set
$Q\subseteq R_{\mathcal B}$.  This gives a local transition operator for one
stratum.  A ranking of rules induces staged saturation by applying local
saturation stratum by stratum.

\begin{definition}[Rule-restricted support additions]
\label{def:restricted-support-additions}
Let $Q\subseteq R_{\mathcal B}$, $\Gamma=\langle S,D\rangle$, and let
$\epsilon$ range over $\{+,-\}$.  The set $\Delta_S^Q(\Gamma)$ is the least
set satisfying the following clauses.
\begin{enumerate}
\item For $\rho\in Q\cap R_P$,
\[
\hspace{-0.8em}
\begin{aligned}
\Gamma\Vdash^+\mathsf{body}(\rho)
&\Rightarrow
(\mathsf{pol}(\mathsf{head}(\rho)),
 \mathsf{atm}(\mathsf{head}(\rho)))
\in\Delta_S^Q(\Gamma),\\
\Gamma\Vdash^-\mathsf{body}(\rho)
&\Rightarrow
(\overline{\mathsf{pol}(\mathsf{head}(\rho))},
 \mathsf{atm}(\mathsf{head}(\rho)))
\in\Delta_S^Q(\Gamma).
\end{aligned}
\]

\item For $r\in Q\cap R_N$,
\[
\Gamma\Vdash^\epsilon\mathsf{body}(r)
\Rightarrow
(\epsilon,c_r)\in\Delta_S^Q(\Gamma).
\]

\item For $r\in Q\cap R_{\mathsf{ex}}$ with
$\mathsf{out}(r)=\mathsf{exempt}(k)$,
\[
\Gamma\Vdash_O r
\Rightarrow
(+,e_k)\in\Delta_S^Q(\Gamma).
\]
\item For $r\in Q\cap R_{\mathsf h}$ with
$\mathsf{out}(r)=\mathsf O(\tau)$,
\[
\Gamma\Vdash_O r
\wedge
\Gamma\Vdash^\epsilon\tau
\Rightarrow
(\epsilon,m_r)\in\Delta_S^Q(\Gamma).
\]

\item For $r\in Q\cap R_{\mathsf h}$ with
$\mathsf{out}(r)=\mathsf F(\tau)$,
\[
\Gamma\Vdash_O r
\wedge
\Gamma\Vdash^\epsilon\tau
\Rightarrow
(\overline{\epsilon},m_r)\in\Delta_S^Q(\Gamma).
\]
\end{enumerate}
\end{definition}

\begin{definition}[Rule-restricted diagnostic additions]
\label{def:restricted-diagnostic-additions}
Let $Q\subseteq R_{\mathcal B}$, $\Gamma=\langle S,D\rangle$, let
$\epsilon$ range over $\{+,-\}$, and put
$S^\star=S\cup\Delta_S^Q(\Gamma)$.  The set $\Delta_D^Q(\Gamma)$ is the
least set satisfying the following clauses.
\begin{enumerate}
\item For $a\in A_{\mathcal B}^{\pm}$,
\[
(+ ,a)\in S^\star
\wedge
(- ,a)\in S^\star
\Rightarrow
\mathsf{conf}(a)\in\Delta_D^Q(\Gamma).
\]

\item For $\rho\in Q\cap R_P$ with
$\nu_\Gamma(\mathsf{body}(\rho))=\bothv$,
\[
\hspace{-3em}
(\epsilon,\mathsf{atm}(\mathsf{head}(\rho)))
\in\Delta_S^Q(\Gamma)
\Rightarrow
\mathsf{flow}(\rho,\mathsf{atm}(\mathsf{head}(\rho)))
\in\Delta_D^Q(\Gamma).
\]

\item For $r\in Q\cap R_N$ with
$\nu_\Gamma(\mathsf{body}(r))=\bothv$,
\[
(\epsilon,c_r)\in\Delta_S^Q(\Gamma)
\Rightarrow
\mathsf{flow}(r,c_r)\in\Delta_D^Q(\Gamma).
\]

\item For $r\in Q\cap R_{\mathsf{ex}}$ with
$\mathsf{out}(r)=\mathsf{exempt}(k)$ and
$\nu_\Gamma(\mathsf{body}(r))=\bothv$,
\[
(+,e_k)\in\Delta_S^Q(\Gamma)
\Rightarrow
\mathsf{flow}(r,e_k)\in\Delta_D^Q(\Gamma).
\]

\item For $r\in Q\cap R_{\mathsf h}$ with
$\mathsf{out}(r)=M(\tau)$ and $\nu_\Gamma(\tau)=\bothv$,
\[
(\epsilon,m_r)\in\Delta_S^Q(\Gamma)
\Rightarrow
\mathsf{flow}(r,m_r)\in\Delta_D^Q(\Gamma).
\]
\end{enumerate}
Support-level conflict detection is global; flow diagnostics are local to
rules in $Q$.
\end{definition}

\begin{definition}[Local transition and local saturation]
\label{def:local-transition-saturation}
For $Q\subseteq R_{\mathcal B}$ and configurations
$\Gamma=\langle S,D\rangle$ and $\Gamma'=\langle S',D'\rangle$,
\[
\Gamma\to_Q\Gamma'
\iff
S'=S\cup\Delta_S^Q(\Gamma)
\wedge
D'=D\cup\Delta_D^Q(\Gamma).
\]
The transition is \emph{strict} iff $\Gamma\neq\Gamma'$.  A configuration
$\Gamma$ is \emph{$Q$-saturated} iff $\Gamma\to_Q\Gamma'$ implies
$\Gamma'=\Gamma$.  When it exists, the unique $Q$-saturated configuration
reachable from $\Gamma$ is denoted by $\mathsf{Sat}_Q(\Gamma)$.
\end{definition}

\begin{definition}[$\lambda$-staged saturation]
\label{def:lambda-staged-saturation}
Let $\lambda:R_{\mathcal B}\to\mathbb N$ be a ranking of the rules of
$\mathcal B$, and set
$R_\lambda^k=\{q\in R_{\mathcal B}\mid \lambda(q)=k\}$.
Let $k_1<\cdots<k_m$ be the indices of the nonempty strata, where
$m$ may be $0$.  Set
\[
\Gamma^0=\Gamma_0,
\qquad
\Gamma^i=\mathsf{Sat}_{R_\lambda^{k_i}}(\Gamma^{i-1})
\quad (1\leq i\leq m).
\]
The resulting configuration is
$\Gamma^{\lambda\ast}_{\mathcal B,\Gamma_0}=\Gamma^m$.
Initial configurations are assumed conflict-closed.
\end{definition}

\subsection{Admissibility and Stable-Read Commitment}
\label{subsec:admissibility-stable-read}

\textsc{Monir} stores support bits monotonically, but some observations used
by rule bodies are not persistent under configuration extension.  A rule may
therefore commit persistent support only after the critical observations on
which the commitment depends have stabilized in lower strata.

For $n\in\mathsf{id}(R_N)$, define the observable state signature of $n$ by
\[
\mathsf{Obs}(n)
=
\{c_n,e_n\}
\cup
\begin{cases}
\{m_n\}, & n\in\mathsf{id}(R_{\mathsf h}),\\
\emptyset, & \text{otherwise.}
\end{cases}
\]
These are the interface atoms whose later addition may affect state tests
concerning $n$.

For each rule $q\in R_{\mathcal B}$, define
\[
\mathsf{Out}(q)=
\begin{cases}
\{\mathsf{atm}(\mathsf{head}(q))\},
& q\in R_P,\\
\{c_q,e_k\},
& q\in R_{\mathsf{ex}},\ \mathsf{out}(q)=\mathsf{exempt}(k),\\
\{c_q,m_q\},
& q\in R_{\mathsf h},\\
\{c_q\},
& q\in R_N\setminus(R_{\mathsf h}\cup R_{\mathsf{ex}}).
\end{cases}
\]

The input signature of a rule is split into a monotone part and a critical
part.  For a factual rule $\rho\in R_P$, let
\[
\mathsf{In}^{\mathsf{mon}}(\rho)
=
\mathsf{VAt}(\mathsf{body}(\rho)),
\qquad
\mathsf{In}^{\mathsf{crit}}(\rho)
=
\mathsf{UAt}(\mathsf{body}(\rho)).
\]
Upper-bounded count expressions are treated as critical because adding support
may invalidate their positive reading.  Counts without a non-vacuous upper
bound are monotone in positive support.
For a normative rule $r\in R_N$, let
\[
\begin{aligned}
\mathsf{In}^{\mathsf{mon}}(r)
=
&\mathsf{VAt}(\mathsf{body}(r))\\
&\cup
\begin{cases}
\mathsf{VAt}(\tau),
& \mathsf{out}(r)=M(\tau),M\in\mathcal M_{\mathsf h},\\
\emptyset,
& \text{otherwise,}
\end{cases}\\
\mathsf{In}^{\mathsf{crit}}(r)
=
&\mathsf{UAt}(\mathsf{body}(r))
\cup
\{e_r\}\\
&\cup
\begin{cases}
\mathsf{UAt}(\tau),
& \mathsf{out}(r)=M(\tau), M\in\mathcal M_{\mathsf h},\\
\emptyset,
& \text{otherwise,}
\end{cases}\\
&\cup
\bigcup_{\state(s,n)\in\mathsf{St}(\mathsf{body}(r))}
\mathsf{Obs}(n).
\end{aligned}
\]
The carrier $e_r$ is critical because exemption support may invalidate the
issued reading of $r$.  State tests are treated conservatively: every state
label of rule $n$ is read through the same observable signature
$\mathsf{Obs}(n)$.

\begin{definition}[Evaluation dependency graph]
\label{def:evaluation-dependency-graph}
The evaluation dependency graph is the edge-labelled graph
$G_{\mathcal B}^{\mathsf{dep}}=(R_{\mathcal B},E,w)$,
where $(q',q)\in E$ iff
\[
\mathsf{Out}(q')
\cap
\bigl(
\mathsf{In}^{\mathsf{mon}}(q)
\cup
\mathsf{In}^{\mathsf{crit}}(q)
\bigr)
\neq\emptyset .
\]
The edge label $w(q',q)\in\{0,1\}$ is defined by
\[
w(q',q)=1
\iff
\mathsf{Out}(q')\cap\mathsf{In}^{\mathsf{crit}}(q)\neq\emptyset,
\]
and $w(q',q)=0$ otherwise.  Edges of weight $0$ are monotone dependencies.
Edges of weight $1$ are critical dependencies.
\end{definition}

\begin{definition}[Admissibility witness]
\label{def:admissibility-witness}
Let $G_{\mathcal B}^{\mathsf{dep}}=(R_{\mathcal B},E,w)$.  A ranking
$\lambda:R_{\mathcal B}\to\mathbb N$ is an \emph{admissibility witness} iff,
for every $(q',q)\in E$,
$\lambda(q')+w(q',q)\leq \lambda(q)$.
A rule base is \emph{admissible} iff it has an admissibility witness.
\end{definition}

\begin{lemma}[Closure of critical inputs]
\label{lem:critical-input-closure}
Let $\lambda$ be an admissibility witness for $\mathcal B$.  If
$\mathsf{Out}(q')\cap\mathsf{In}^{\mathsf{crit}}(q)\neq\emptyset$ ,
then $\lambda(q')<\lambda(q)$.  Hence, after all lower strata of $q$ have
been saturated, no rule in the current or a later stratum can add a carrier
in $\mathsf{In}^{\mathsf{crit}}(q)$.
\end{lemma}

\begin{proof}
The intersection gives an edge $(q',q)$ of weight $1$.  By admissibility,
$\lambda(q')+1\leq\lambda(q)$.
Thus $\lambda(q')<\lambda(q)$.  The final claim follows from the definition
of $\mathsf{Out}$ and $G_{\mathcal B}^{\mathsf{dep}}$.
\end{proof}

\begin{proposition}[Graph characterization of admissibility]
\label{prop:admissibility-graph-characterization}
A rule base is admissible iff no directed cycle of
$G_{\mathcal B}^{\mathsf{dep}}$ contains an edge of weight $1$.
\end{proposition}

\begin{proof}
If an admissibility witness exists, ranks are nondecreasing along every edge
and strictly increase along every weight-$1$ edge.  Hence no directed cycle
contains a weight-$1$ edge.
Conversely, suppose no directed cycle contains a weight-$1$ edge.  Then every
strongly connected component contains only weight-$0$ edges.  Contract the
strongly connected components.  The resulting graph is acyclic.  Assign to
each component the maximum weight of a path ending at it, and give each rule
the rank of its component.  This ranking satisfies the witness condition.
\end{proof}

\begin{definition}[Least-barrier witness]
\label{def:least-barrier-witness}
Let $\mathcal B$ be admissible.  For a directed path
$\pi=q_0\to q_1\to\cdots\to q_m$, including the length-zero path $q_0$, let $w(\pi)=\sum_{i=1}^{m} w(q_{i-1},q_i)$,
\[
\lambda_{\min}(q)
=
\max\{w(\pi)\mid \pi \text{ is a directed path ending at } q\}.
\]
The maximum is finite by admissibility; call $\lambda_{\min}$ the
\emph{least-barrier witness}.
\end{definition}

\begin{proposition}[Minimality of the least-barrier witness]
\label{prop:least-barrier-witness}
If $\mathcal B$ is admissible, then $\lambda_{\min}$ is an admissibility
witness.  Moreover, for every admissibility witness $\lambda$ and every
$q\in R_{\mathcal B}$,
$\lambda_{\min}(q)\leq \lambda(q)$.
\end{proposition}

\begin{proof}
For any edge $(q',q)$, extending every path ending at $q'$ by this edge gives
$\lambda_{\min}(q)\geq\lambda_{\min}(q')+w(q',q)$.  Thus
$\lambda_{\min}$ is a witness.  For any witness $\lambda$, summing the witness
inequalities along a path $\pi$ ending at $q$ yields
$\lambda(q)\geq w(\pi)$.  Maximizing over such paths gives
$\lambda(q)\geq\lambda_{\min}(q)$.
\end{proof}

Let $\mathcal B$ be admissible and let $\lambda_{\min}$ be its least-barrier
witness.  The default \textsc{Monir}-core saturated result from $\Gamma_0$ is
$\Gamma^{\mathsf{st}}_{\mathcal B,\Gamma_0}
=
\Gamma^{\lambda_{\min}\ast}_{\mathcal B,\Gamma_0}$.
Unless explicitly stated otherwise, \emph{saturation} in \textsc{Monir}-core means default staged saturation.

\subsection{Policy-Parametric Aggregation}
\label{subsec:policy-parametric-aggregation}

The staged transition semantics stops at saturated support configurations.
It derives support bits, diagnostics, rule states, and issued outputs, but
leaves compliance-level reporting to read-only aggregation policies.
Aggregation does not change support configurations, admissibility witnesses,
or staged saturation.

An \emph{output item} is a pair in $\mathcal N\times\Omega$.  A
\emph{report item} is a triple $(k,n,z)$, where $k$ is a report label,
$n\in\mathcal N$, and $z$ is an output or report payload.
For a support configuration $\Gamma$, define the issued outputs
\[
\mathsf{Issued}_{\mathcal B}(\Gamma)
=
\{(\mathsf{id}(r),\mathsf{out}(r))
  \mid
  r\in R_N,\ \Gamma\Vdash_O r\}.
\]

\begin{definition}[Aggregation policy]
\label{def:aggregation-policy}
An \emph{aggregation policy} is a tuple
$\Pi=
\langle
\mathsf{class}_{\Pi},
\mathsf{outconf}_{\Pi},
\mathsf{verdict}_{\Pi}
\rangle$ .
For each finite rule base $\mathcal B$,
$\mathsf{class}_{\Pi}^{\mathcal B}$ maps saturated configurations to finite
sets of report items.  The relation
$\mathsf{outconf}_{\Pi}^{\mathcal B}$ is a finite relation of triples
$(h,x,y)$, where $h$ is an output-conflict label and
\[
x,y\in
\{(\mathsf{id}(r),\mathsf{out}(r))\mid r\in R_N\}.
\]
The function $\mathsf{verdict}_{\Pi}^{\mathcal B}$ maps finite sets of report
items, output conflicts, and transition diagnostics to finite sets of verdict
labels.
A policy is \emph{polynomial-time} if membership in
$\mathsf{class}_{\Pi}^{\mathcal B}(\Gamma)$, membership in
$\mathsf{outconf}_{\Pi}^{\mathcal B}$, and membership in
$\mathsf{verdict}_{\Pi}^{\mathcal B}(A,C,D)$ are computable in polynomial time
in the explicit size of $\mathcal B$ and $\Gamma$.

For a saturated configuration $\Gamma=\langle S,D\rangle$, let
$A=\mathsf{class}_{\Pi}^{\mathcal B}(\Gamma)$
and
\[
C=
\{(h,x,y)\in\mathsf{outconf}_{\Pi}^{\mathcal B}
  \mid
  x,y\in\mathsf{Issued}_{\mathcal B}(\Gamma),\
  x\neq y\}.
\]
The report induced by $\Pi$ is
\[
\mathsf{Rep}_{\Pi}^{\mathcal B}(\Gamma)
=
\langle A,C,\mathsf{verdict}_{\Pi}^{\mathcal B}(A,C,D)\rangle .
\]
Here $D$ records transition diagnostics, while $C$ records conflicts between
issued outputs under $\Pi$.
\end{definition}

\begin{definition}[Reference diagnostic policy]
\label{def:reference-diagnostic-policy}
The \emph{reference diagnostic policy}
$\Pi_0=
\langle
\mathsf{class}_{0},
\mathsf{outconf}_{0},
\mathsf{verdict}_{0}
\rangle$
uses report labels
$\mathsf{viol}$, $\mathsf{unres}$, $\mathsf{perm}$, $\mathsf{rec}$,
$\mathsf{contest}$, and $\mathsf{ctd}$; conflict labels
$\mathsf{hard}$, $\mathsf{soft}$, and $\mathsf{warn}$; and verdict labels
$\mathsf{compliant}$, $\mathsf{noncompliant}$, $\mathsf{undetermined}$,
$\mathsf{conflicted}$, and $\mathsf{warning}$.

For a staged saturated configuration $\Gamma$,
$\mathsf{class}_{0}^{\mathcal B}(\Gamma)$ is generated by:
\[
\small
\begin{array}{@{}l|l@{}}
\text{item} & \text{condition}\\
\hline
(\mathsf{viol},n,M(\tau))
&
\begin{array}[t]{@{}l@{}}
(n,M(\tau))\in\mathsf{Issued}_{\mathcal B}(\Gamma)\wedge M\in\mathcal M_{\mathsf h}\\
{}\wedge \Gamma\Vdash_S\state(\mathsf{violated},n)
\end{array}
\\[0.45em]

(\mathsf{unres},n,M(\tau))
&
\begin{array}[t]{@{}l@{}}
(n,M(\tau))\in\mathsf{Issued}_{\mathcal B}(\Gamma)\wedge M\in\mathcal M_{\mathsf h}\\
{}\wedge \Gamma\Vdash_S\state(\mathsf{unknown},n)
\end{array}
\\[0.45em]

(\mathsf{unres},n,\mathsf{pending})
&
\begin{array}[t]{@{}l@{}}
\Gamma\Vdash_S\state(\mathsf{pending},n)\\
{}\wedge \neg\Gamma\Vdash_S\state(\mathsf{exempted},n)
\end{array}
\\[0.45em]

(\mathsf{contest},n,\mathsf{contested})
&
\begin{array}[t]{@{}l@{}}
\Gamma\Vdash_S\state(\mathsf{contested},n)
\end{array}
\\[0.45em]

(\mathsf{perm},n,\mathsf P(\tau))
&
\begin{array}[t]{@{}l@{}}
(n,\mathsf P(\tau))\in\mathsf{Issued}_{\mathcal B}(\Gamma)
\end{array}
\\[0.45em]

(\mathsf{rec},n,M(\tau))
&
\begin{array}[t]{@{}l@{}}
(n,M(\tau))\in\mathsf{Issued}_{\mathcal B}(\Gamma)\wedge M\in\{\mathsf R,\mathsf{NR}\}
\end{array}
\\[0.45em]

(\mathsf{ctd},n,\omega)
&
\begin{array}[t]{@{}l@{}}
(n,\omega)\in\mathsf{Issued}_{\mathcal B}(\Gamma)\\
{}\wedge \exists m\,
\bigl(\state(\mathsf{violated},m)\in\mathsf{St}(\mathsf{body}(n))\bigr)
\end{array}
\end{array}
\]

The relation $\mathsf{outconf}_{0}^{\mathcal B}$ is the least symmetric
labelled relation generated by the following rows, for rule identifiers
$n,m$ and factual targets $\tau$, whenever both displayed output items occur
in $\{(\mathsf{id}(r),\mathsf{out}(r))\mid r\in R_N\}$:
\[
\small
\begin{array}{l|l|l}
\text{label} & \text{first output} & \text{second output}\\
\hline
\mathsf{hard} & (n,\mathsf O(\tau)) & (m,\mathsf F(\tau))\\
\mathsf{hard} & (n,\mathsf P(\tau)) & (m,\mathsf F(\tau))\\
\mathsf{soft} & (n,\mathsf R(\tau)) & (m,\mathsf{NR}(\tau))\\
\mathsf{warn} & (n,\mathsf O(\tau)) & (m,\mathsf{NR}(\tau))\\
\mathsf{warn} & (n,\mathsf F(\tau)) & (m,\mathsf R(\tau))\\
\mathsf{warn} & (n,\mathsf P(\tau)) & (m,\mathsf{NR}(\tau)).
\end{array}
\]

For a set $A$ of report items and a set $C$ of detected output conflicts,
write
\[
A_k=\{(n,z)\mid (k,n,z)\in A\},
\
C_h=\{(x,y)\mid (h,x,y)\in C\}.
\]
The verdict function $\mathsf{verdict}_{0}^{\mathcal B}(A,C,D)$ returns each
verdict whose trigger condition holds:
\[
\small
\begin{array}{l|l}
\text{verdict} & \text{trigger}\\
\hline
\mathsf{noncompliant}
&
A_{\mathsf{viol}}\neq\emptyset
\\[0.4em]

\mathsf{undetermined}
&
A_{\mathsf{unres}}\neq\emptyset
\\[0.4em]

\mathsf{conflicted}
&
\begin{array}{@{}l@{}}
C_{\mathsf{hard}}\neq\emptyset
\vee \exists a\,\mathsf{conf}(a)\in D \vee \\
{}A_{\mathsf{contest}}\neq\emptyset
\vee \exists q,a\,\mathsf{flow}(q,a)\in D
\end{array}
\\[0.4em]

\mathsf{warning}
&
C_{\mathsf{soft}}\neq\emptyset
\vee
C_{\mathsf{warn}}\neq\emptyset
\\[0.4em]

\mathsf{compliant}
&
\begin{array}{@{}l@{}}
A_{\mathsf{viol}}\cup A_{\mathsf{unres}}\cup A_{\mathsf{contest}}=\emptyset
\\
{}\wedge C_{\mathsf{hard}}=\emptyset
\wedge\nexists a\,\mathsf{conf}(a)\in D
\\
{}\wedge\nexists q,a\,\mathsf{flow}(q,a)\in D
\end{array}
\end{array}
\]
The verdict is diagnostic: several verdict labels may be returned for the same
staged saturated configuration.
\end{definition}

Contrary-to-duty outputs are marked by $\mathsf{ctd}$ but do not discharge the
primary violation.  Likewise, transition diagnostics are explanatory rather
than retracting: conflict and flow records account for contested or
conflict-dependent outputs without cancelling issued outputs.

\section{Semantic Properties}
\label{sec:semantic-properties}

This section records persistence properties, deterministic staged evaluation,
and the complexity of fixed-evidence and completion-based compliance queries.

\subsection{Persistent Storage}
\label{subsec:persistent-storage}

For configurations
$\Gamma=\langle S,D\rangle$ and $\Gamma'=\langle S',D'\rangle$, write
$\Gamma\sqsubseteq\Gamma'$ iff
$S\subseteq S'\wedge D\subseteq D'$.
By Definition~\ref{def:local-transition-saturation}, every local transition
extends both support and diagnostic sets.  Thus stored information is
monotone under $\sqsubseteq$.

Storage monotonicity concerns stored support and diagnostics.  It does not
make derived observations persistent: expression readings, rule states,
issued outputs, and report labels are computed from the current configuration.

\begin{proposition}[Non-persistence of derived observations]
\label{prop:derived-nonpersistence}
The following failures occur under configuration extension.
\begin{enumerate}
\item There are $\Gamma\sqsubseteq\Gamma'$ and a factual condition $\kappa$
such that $\Gamma\Vdash^+\kappa$ but $\Gamma'\not\Vdash^+\kappa$.
\item There are $\Gamma\sqsubseteq\Gamma'$ and a normative rule $r\in R_N$
such that $\Gamma\Vdash_O r$ but $\Gamma'\not\Vdash_O r$.
\end{enumerate}
\end{proposition}

\begin{proof}
For 1, take a rule base with $p\in\mathcal P$ and set
$\kappa=\countop_{[0,0]}(p)$,
$\Gamma=\langle\emptyset,\emptyset\rangle$, and
$\Gamma'=\langle\{(+,p)\},\emptyset\rangle$.
Then $\Gamma\sqsubseteq\Gamma'$,
$\Gamma\Vdash^+\kappa$, and
$\Gamma'\not\Vdash^+\kappa$.
For 2, take a rule base containing
$r=n:\top\Rightarrow\mathsf O(p)$,
$\Gamma=\langle\{(+,c_r)\},\emptyset\rangle$, and
$\Gamma'=\langle\{(+,c_r),(+,e_r)\},\emptyset\rangle$.
Then $r$ is issued in $\Gamma$ but not in $\Gamma'$.
\end{proof}

\subsection{Polynomial Evaluation and Completion Queries}
\label{subsec:polynomial-evaluation-completion-queries}

We analyze fixed-evidence evaluation and completion-based verdict queries.  The
evaluation result is stated for an arbitrary ranking $\lambda$ and specialized
to the default admissible semantics by taking $\lambda=\lambda_{\min}$.
\begin{theorem}[Polynomial staged evaluation]
\label{thm:lambda-staged-poly}
For every finite well-formed rule base $\mathcal B$, initial configuration
$\Gamma_0$, and ranking $\lambda:R_{\mathcal B}\to\mathbb N$,
$\Gamma^{\lambda\ast}_{\mathcal B,\Gamma_0}$ exists and is unique.  It is
computable in polynomial time in the explicit size of
$\mathcal B$, $\Gamma_0$, and $\lambda$.
\end{theorem}

\begin{proof}
Let $Q$ be fixed.  By
Definition~\ref{def:local-transition-saturation}, the successor of any
configuration under $\to_Q$ is uniquely determined.  Every strict transition
adds at least one element of
$\mathsf{Bit}(\mathcal B)\cup\mathsf{Diag}(\mathcal B)$ and never removes
one.  Hence every $Q$-run has length at most
$|\mathsf{Bit}(\mathcal B)|+|\mathsf{Diag}(\mathcal B)|$.

Each successor is computed by evaluating the finitely represented rule bodies,
targets, and diagnostic clauses in Definitions~\ref{def:restricted-support-additions}
and~\ref{def:restricted-diagnostic-additions}.  This is polynomial in the
explicit input size.  Therefore $\mathsf{Sat}_Q(\Gamma)$ exists, is unique,
and is polynomial-time computable.  Applying this argument to the finitely many
nonempty strata of $\lambda$ proves the claim.
\end{proof}

A \emph{fixed-evidence query item} is a support bit, diagnostic record,
rule-state atom, issued-output item, report item, or verdict label.
Consequently, for any polynomial-time aggregation policy $\Pi$, membership of
a fixed-evidence query item in the $\lambda$-staged saturated result is
decidable in polynomial time: compute
$\Gamma^{\lambda\ast}_{\mathcal B,\Gamma_0}$ and then perform the corresponding
read-only membership test.

\begin{corollary}[Fixed-evidence membership for default admissible semantics]
\label{cor:fixed-evidence-membership-default}
Fix a polynomial-time aggregation policy $\Pi$.  Given an admissible rule base
$\mathcal B$, an initial configuration $\Gamma_0$, and a fixed-evidence query
item, membership of the item in the default saturated result
$\Gamma^{\mathsf{st}}_{\mathcal B,\Gamma_0}$ is decidable in polynomial time.
\end{corollary}

\begin{proof}[Proof sketch]
Construct $G_{\mathcal B}^{\mathsf{dep}}$ from the finite input/output
signatures.  By
Proposition~\ref{prop:admissibility-graph-characterization}, admissibility is
checked by testing whether some SCC contains an internal weight-$1$ edge.  For
an admissible rule base, $\lambda_{\min}$ is obtained by a longest-path
computation on the SCC condensation DAG whose edge weights are inherited from
$G_{\mathcal B}^{\mathsf{dep}}$.  The claim then follows from
Theorem~\ref{thm:lambda-staged-poly}.
\end{proof}

Fixed-evidence membership assumes that the initial support configuration is
fixed.  Completion-based queries instead quantify over possible completions of
currently unsupported factual atoms.

\begin{definition}[Evidence completions and completion-based membership]
\label{def:evidence-completions}
Let $\Gamma_0=\langle S_0,D_0\rangle$, and let
$U\subseteq\mathcal P$ be a finite set of completion atoms such that
$(+,p),(-,p)\notin S_0$ for all $p\in U$.
A \emph{total consistent completion} over $U$ is a set $E$ such that, for each
$p\in U$, exactly one of $(+,p)$ and $(-,p)$ belongs to $E$.  Let
$\mathsf{Comp}(U)$ be the set of all such completions.  For
$E\in\mathsf{Comp}(U)$, define
$\Gamma_0\oplus E=\langle S_0\cup E,D_0\rangle$ .

Given an admissible rule base $\mathcal B$, an aggregation policy $\Pi$, and
a verdict label $v$, \emph{possible verdict membership} asks whether $v$ is
returned by $\Pi$ from
$\Gamma^{\mathsf{st}}_{\mathcal B,\Gamma_0\oplus E}$
for some $E\in\mathsf{Comp}(U)$.  \emph{Necessary verdict membership} asks
whether $v$ is returned by $\Pi$ from
$\Gamma^{\mathsf{st}}_{\mathcal B,\Gamma_0\oplus E}$
for every $E\in\mathsf{Comp}(U)$.
\end{definition}

\begin{theorem}[Completion-based compliance membership]
\label{thm:completion-membership-complexity}
For every polynomial-time aggregation policy, possible verdict membership is in
NP and necessary verdict membership is in coNP.  Moreover, for a fixed
polynomial-time aggregation policy, possible $\mathsf{noncompliant}$
membership is NP-complete and necessary $\mathsf{compliant}$ membership is
coNP-complete, even for admissible rule bases.
\end{theorem}

\begin{proof}
Possible membership is in NP: guess $E\in\mathsf{Comp}(U)$, compute
$\Gamma^{\mathsf{st}}_{\mathcal B,\Gamma_0\oplus E}$ by
Corollary~\ref{cor:fixed-evidence-membership-default}, and evaluate the
verdict in polynomial time.  Necessary membership is in coNP because its
complement has an NP certificate: a completion
$E\in\mathsf{Comp}(U)$ for which the queried verdict label is not returned.

For hardness, reduce SAT.  Let
$\varphi=C_1\wedge\cdots\wedge C_m$ be a CNF formula over variables
$x_1,\ldots,x_n$, and put $U=\{x_1,\ldots,x_n\}$.  A completion assigns each
$x_i$ exactly one polarity.  Add atoms $d_1,\ldots,d_m,g,ok$.  For every
literal occurrence $\ell_{j,t}$ in $C_j$, add
$a_{j,t}:\ell_{j,t}\leadsto d_j$, and add
\[
a_0:d_1\wedge\cdots\wedge d_m\leadsto g,\quad
a_1:g\leadsto\neg ok,\quad
r_\varphi:g\Rightarrow\mathsf O(ok).
\]
Use the fixed polynomial-time policy $\Pi_{\mathsf{sat}}$ that reports a
violation item for every issued hard output in state
$\mathsf{violated}$, returns $\mathsf{noncompliant}$ iff such an item exists,
and returns $\mathsf{compliant}$ iff no such item exists.  Under a completion, $g$ is derived exactly
when all clauses are satisfied.  In that case $r_\varphi$ is issued and
$a_1$ derives negative support for $ok$, so the obligation $\mathsf O(ok)$ is
violated.  Thus possible $\mathsf{noncompliant}$ membership holds iff
$\varphi$ is satisfiable.  The dependency graph of the construction is
acyclic, hence admissible.  Therefore necessary $\mathsf{compliant}$
membership holds iff $\varphi$ is unsatisfiable, giving coNP-hardness.
\end{proof}

\section{Executable Realization}
\label{sec:executable-realization}

This section describes the ASP backend used to execute the reference
\textsc{Monir}-core semantics.  It introduces no separate semantics.  A
backend run takes structured \textsc{Monir} records, instantiates the relevant
ASP facts and rule templates, calls clingo on selected rule blocks, and
projects the resulting atoms back to support, diagnostic, output, report, and
verdict records.

The backend uses two execution strategies with explicit invariance conditions.
Topology-respecting saturation evaluates dependency components by ready
frontiers.  Demand reasoning evaluates a backward slice sufficient for a
finite query under complete seed metadata.  Both strategies use the
input/output signatures and dependency edges of
Definition~\ref{def:evaluation-dependency-graph}, together with SCC components
and optional scheduling cost indexes.

Implementation-specific proposition handlers are outside the core semantics.
In an ADAS instantiation, numeric comparisons, symbolic matching, and static
temporal-pattern checks can be handled as deterministic preprocessing steps or
as ASP-side handlers that emit ordinary \textsc{Monir} support records.  Such
handlers preserve the core semantics only when their outputs are fixed support
records before the corresponding saturation step.  They therefore extend the
modelling interface, not Definitions~\ref{def:restricted-support-additions}
--\ref{def:lambda-staged-saturation}.

\subsection{Topology-Respecting Parallel Saturation}
\label{subsec:exec-parallel-saturation}

Let
$\mathcal C=\mathsf{SCC}(G_{\mathcal B}^{\mathsf{dep}})$, and let
$R(C)$ be the set of rules in a component $C$.  Write
$\mathsf{CDAG}_{\mathcal B}$ for the condensation graph.  Components are the
atomic scheduling units; they are not split by the scheduler.

For $C\in\mathcal C$, let $\mathsf{CPred}(C)$ be the set of predecessors of
$C$ in $\mathsf{CDAG}_{\mathcal B}$.  For $X\subseteq\mathcal C$, define
\[
\mathsf{Front}(X)=
\{C\in\mathcal C\setminus X\mid \mathsf{CPred}(C)\subseteq X\}.
\]
Only components in the same ready frontier may be evaluated concurrently.
Weight-$1$ edges are stage barriers by admissibility; weight-$0$ edges still
give topological precedence because they may transmit support to later
components.

Batching is an implementation choice.  A batch is a subset of one ready
frontier, and each call to \textsc{PlanBatches} returns a partition of that
frontier.  A scheduler may use a cost estimate
$\mu(C)=|R(C)|+\xi\cdot n_{\mathsf{num}}(C)$,
where $n_{\mathsf{num}}(C)$ counts numeric-comparison handlers and
$\xi>0$ is a scheduling parameter.  This score affects only job granularity.

\begin{algorithm}[t]
\caption{\textsc{ParallelSaturate}}
\label{alg:exec-parallel-saturate}
\begin{algorithmic}[1]
\REQUIRE rule base $\mathcal B$, initial configuration $\Gamma_0$,
component set $\mathcal C$, condensation graph $\mathsf{CDAG}_{\mathcal B}$,
target batch cost $\beta$
\STATE $\Gamma\leftarrow\Gamma_0$; $X\leftarrow\emptyset$
\WHILE{$X\neq\mathcal C$}
  \STATE $F\leftarrow\mathsf{Front}(X)$
  \STATE $(U_1,\ldots,U_m)\leftarrow\textsc{PlanBatches}(F,\beta)$
  \STATE in parallel compute
  $\Gamma_i\leftarrow\mathsf{Sat}_{R(U_i)}(\Gamma)$
  for $i=1,\ldots,m$
  \STATE $\Gamma\leftarrow\textsc{Merge}(\Gamma_1,\ldots,\Gamma_m)$
  \STATE $X\leftarrow X\cup F$
\ENDWHILE
\RETURN $\Gamma$
\end{algorithmic}
\end{algorithm}
Here $R(U_i)=\bigcup_{C\in U_i}R(C)$.  If
$\Gamma_i=\langle S_i,D_i\rangle$, put
$S^\cup=\bigcup_i S_i$,
$D^\cup=\bigcup_i D_i$,
and
\[
\mathsf{Conf}(S)=
\{\mathsf{conf}(a)\mid
(\mathord{+},a)\in S,\;(\mathord{-},a)\in S\}.
\]
The merge operation is
\[
\textsc{Merge}(\Gamma_1,\ldots,\Gamma_m)
=
\langle S^\cup, D^\cup\cup\mathsf{Conf}(S^\cup)\rangle .
\]
The conflict closure is needed because opposite support for the same
carrier may be produced by different batches.

\begin{proposition}[Batching invariance]
\label{prop:exec-batching-invariance}
For fixed $\mathcal B$, $\Gamma_0$, and
$\mathcal C=\mathsf{SCC}(G_{\mathcal B}^{\mathsf{dep}})$,
Algorithm~\ref{alg:exec-parallel-saturate} returns the same configuration for
all choices of \textsc{PlanBatches}, provided each call partitions the current
ready frontier and \textsc{Merge} adds support conflicts visible after batch
union.
\end{proposition}

\begin{proof}[Proof sketch]
At each round, $\mathsf{Front}(X)$ is an antichain in
$\mathsf{CDAG}_{\mathcal B}$.  Hence distinct frontier components have no
producer-to-reader dependency in $G_{\mathcal B}^{\mathsf{dep}}$.  All batches
therefore read the same predecessor-closed configuration, and their support
additions commute under union.

Flow diagnostics are local to the rules evaluated in a batch.  If a support
item produced by one batch could affect a flow diagnostic in another, the
definition of $G_{\mathcal B}^{\mathsf{dep}}$ would give an edge between the
two frontier components, contradicting antichainness.  The only diagnostic
that may first become visible after union is a support conflict, which is
exactly added by $\mathsf{Conf}(S^\cup)$.

Thus each round is independent of the chosen partition of the frontier.
Induction over the rounds gives the claim.
\end{proof}

\subsection{Demand Reasoning}
\label{subsec:exec-demand-reasoning}

Demand reasoning is a query-time optimization.  It introduces no separate
semantics: the reference result remains the default staged saturation of
Section~\ref{subsec:admissibility-stable-read}.  Given a finite demand, the
backend evaluates an over-approximate backward slice and projects the demanded
objects.

A demand is a finite set $\mathcal D$ of query objects.  Query objects may be
support bits, diagnostic records, rule-state atoms, issued-output items,
report items, or verdict labels.  Output conflicts are not primitive query
objects; when needed, they are queried through report items or verdict labels.
All demands below are assumed to be finite and to contain only objects for
which the corresponding membership test and seed set are defined.

For a carrier $a$, define the producer index
\[
\mathsf{Pr}(a)=\{q\in R_{\mathcal B}\mid a\in\mathsf{Out}(q)\},
\
\mathsf{Pr}(X)=\bigcup_{a\in X}\mathsf{Pr}(a).
\]
Define the seed map for core
query objects by
\[
\begin{aligned}
&\sigma_{\Pi}^{\mathcal B}((\epsilon,a))
=\mathsf{Pr}(a),\quad
\sigma_{\Pi}^{\mathcal B}(\mathsf{conf}(a))
=\mathsf{Pr}(a),\\
&\sigma_{\Pi}^{\mathcal B}(\mathsf{flow}(q,a))
=\{q\},\quad
\sigma_{\Pi}^{\mathcal B}((n,\omega))
=\mathsf{Pr}(\{c_n,e_n\}),\\
&\sigma_{\Pi}^{\mathcal B}(\state(s,n))
=\mathsf{Pr}(\mathsf{Obs}(n)).
\end{aligned}
\]
For report items and verdict labels in $\mathcal D$, the policy supplies
finite seed sets.  The seed interface is \emph{complete for} $\mathcal D$ if,
for every policy-level query object $d\in\mathcal D$, membership of $d$ is
invariant under changes to rules outside
$\mathsf{Back}(\sigma_{\Pi}^{\mathcal B}(d))$, whenever the core records
generated by that backward closure are fixed.  For global verdicts such as
$\mathsf{compliant}$ under $\Pi_0$, the complete seed set may be all of
$R_{\mathcal B}$.

Put
$\sigma(\mathcal D)
=
\bigcup_{d\in\mathcal D}\sigma_{\Pi}^{\mathcal B}(d)$.
Let $E$ be the edge set of
$G_{\mathcal B}^{\mathsf{dep}}$ from
Definition~\ref{def:evaluation-dependency-graph}.  For
$Q\subseteq R_{\mathcal B}$, define
\[
\begin{aligned}
&\mathsf{Pred}(Q)=
\{q'\mid \exists q\in Q,\ (q',q)\in E\}, \\
&\mathsf{Back}(Q)
=
\bigcap
\{U\subseteq R_{\mathcal B}\mid
Q\subseteq U,\ \mathsf{Pred}(U)\subseteq U\}.
\end{aligned}
\]
A rule set $R_{\mathcal D}$ is a \emph{demand-closed slice} for
$\mathcal D$ if
$\mathsf{Back}(\sigma(\mathcal D))
\subseteq
R_{\mathcal D}$
Thus slices may over-approximate the minimal backward closure.

For
$\mathsf{Rep}_{\Pi}^{\mathcal B}(\Gamma)=\langle A,C,V\rangle$
and $\Gamma=\langle S,D\rangle$, define
\[
\begin{aligned}
\mathsf{Hold}_{\Pi}^{\mathcal B}(\Gamma)
&=
S
\cup D
\cup \mathsf{Issued}_{\mathcal B}(\Gamma)
\cup A
\cup V\\
&\hspace{-2em}\cup
\{\state(s,n)\mid
s\in\mathcal S, n\in\mathsf{id}(R_N),
\Gamma\Vdash_S\state(s,n)\}.
\end{aligned}
\]
The demanded answer from $\Gamma$ is
$\mathsf{Hold}_{\Pi}^{\mathcal B}(\Gamma)\cap\mathcal D$.

\begin{algorithm}[t]
\caption{\textsc{DemandReason}}
\label{alg:exec-demand-reason}
\begin{algorithmic}[1]
\REQUIRE admissible rule base $\mathcal B$, initial configuration
$\Gamma_0=\langle S_0,D_0\rangle$, finite demand $\mathcal D$,
demand-closed slice $R_{\mathcal D}$
\IF{$\mathcal D=\emptyset$}
  \RETURN $\emptyset$
\ENDIF
\STATE $\Gamma\leftarrow
\langle S_0,D_0\cup\mathsf{Conf}(S_0)\rangle$
\FOR{each nonempty stratum $k$ of $\lambda_{\min}$ in increasing order}
  \STATE $Q\leftarrow R_{\lambda_{\min}}^k\cap R_{\mathcal D}$
  \IF{$Q\neq\emptyset$}
    \STATE $\Gamma\leftarrow\mathsf{Sat}_Q(\Gamma)$
  \ENDIF
\ENDFOR
\STATE $\mathsf{Ans}_{\mathcal D}\leftarrow
\mathsf{Hold}_{\Pi}^{\mathcal B}(\Gamma)\cap\mathcal D$
\RETURN $\mathsf{Ans}_{\mathcal D}$
\end{algorithmic}
\end{algorithm}

Each local saturation call in Algorithm~\ref{alg:exec-demand-reason} may be
implemented by Algorithm~\ref{alg:exec-parallel-saturate}.  This is an
implementation choice; the correctness statement below uses the reference
operator $\mathsf{Sat}_Q$ from
Definition~\ref{def:local-transition-saturation}.

\begin{proposition}[Projection exactness of demand reasoning]
\label{prop:exec-demand-projection-exactness}
Let $\Gamma_0=\langle S_0,D_0\rangle$ and
$\widehat{\Gamma}_0=
\langle S_0,D_0\cup\mathsf{Conf}(S_0)\rangle$.
Assume that $R_{\mathcal D}$ is demand-closed for $\mathcal D$ and that the
policy seed interface is complete for $\mathcal D$.  If
Algorithm~\ref{alg:exec-demand-reason} returns
$\mathsf{Ans}_{\mathcal D}$, then
\[
\mathsf{Ans}_{\mathcal D}
=
\mathsf{Hold}_{\Pi}^{\mathcal B}
(\Gamma^{\mathsf{st}}_{\mathcal B,\widehat{\Gamma}_0})
\cap\mathcal D .
\]
If initial configurations are conflict-closed, then
$\widehat{\Gamma}_0=\Gamma_0$.
\end{proposition}

\begin{proof}
Let $\Gamma_{\mathcal D}$ be the configuration computed by
Algorithm~\ref{alg:exec-demand-reason}. Since $R_{\mathcal D}$ contains
$\mathsf{Back}(\sigma(\mathcal D))$, it is predecessor-closed for the seed
rules of the demand.  Hence no rule outside $R_{\mathcal D}$ can produce a
carrier read by a rule inside the backward closure without inducing a
dependency edge into that closure.  At each stratum $k$, saturating
$R_{\lambda_{\min}}^k\cap R_{\mathcal D}$ therefore agrees with full staged
saturation on the core records reachable from the demand seeds.

Induction over the strata of $\lambda_{\min}$ gives agreement between
$\Gamma_{\mathcal D}$ and
$\Gamma^{\mathsf{st}}_{\mathcal B,\widehat{\Gamma}_0}$ on the demanded
support bits, diagnostics, rule states, and issued outputs.  For demanded
report items and verdict labels, agreement follows from completeness of the
policy seed interface.  Therefore both configurations induce the same
membership status for every object in $\mathcal D$, which is the displayed
projection equality.
\end{proof}

\bibliography{aaai2026}

\end{document}